\documentclass[letterpaper, 10 pt, conference]{ieeeconf}
\usepackage{graphicx} 
\IEEEoverridecommandlockouts
\usepackage{times}
\usepackage{multicol}
\usepackage[bookmarks=true]{hyperref}

\usepackage{hyperref}
\usepackage{graphicx}		
\usepackage{wrapfig}
\usepackage[format=plain,font=footnotesize,labelfont=bf,labelsep=period]{caption}
\usepackage[export]{adjustbox}
\usepackage{bm}
\usepackage{subcaption}
\usepackage{balance}
\usepackage{dsfont}

\usepackage{booktabs} 
\usepackage{amsmath}
\usepackage{mathrsfs}
\usepackage{amssymb}

\usepackage{amsthm}
\usepackage{amsthm}
\usepackage{amsfonts}
\usepackage{mathrsfs}  
\usepackage{mathtools}
 \usepackage{tikz} 
\usetikzlibrary{positioning,arrows.meta,shadows.blur,fit}
\usepackage{tikzducks}
\usepackage{mwe} 
\usepackage{mathtools}
\usepackage{relsize}
\usepackage{bbm}
\usepackage{cite}
\usepackage{algorithm}
\usepackage{algpseudocode}

\usepackage{pdfsync}
\usepackage[normalem]{ulem}
\usepackage{paralist}	
\usepackage[space]{grffile} 

\usepackage[usenames,dvipsnames]{xcolor}
\usepackage{centernot}

\usepackage{graphicx}   
\usepackage{environ}    

\newtheorem{theorem}{Theorem}
\newtheorem*{theorem*}{Theorem}

\newtheorem{lemma}{Lemma}
\newtheorem*{lemma*}{Lemma}

\theoremstyle{definition}
\newtheorem{definition}{Definition}
\newtheorem{remark}{Remark}

\newtheorem*{proposition*}{Proposition}

\newcommand{\norm}[1]{\left\lVert#1\right\rVert}

\newcommand{\newsec}[1]{\vspace{0.1cm} \noindent \textbf{#1} }

\newcommand{\R}{\mathbb{R}}

\definecolor{darkblue}{RGB}{0,0,102}
\definecolor{lightblue}{RGB}{77,77,148}

\definecolor{gold}{RGB}{234, 170, 0}
\definecolor{metallic_gold}{RGB}{139, 111, 78}

\newcommand{\mb}[1]{\mathbf{ #1 }}

\newcommand{\lmat }{\begin{bmatrix}}
\newcommand{\rmat}{\end{bmatrix}}

\usepackage{dsfont}
 
\usepackage{svg}
\usepackage{caption}

\makeatletter
\long\def\@makefntext#1{%
  \noindent
  #1}
\makeatother

\IEEEoverridecommandlockouts
\usepackage{mathtools}
\usepackage{xcolor}
\usepackage{url}

\newcommand{\Sc}{\mathcal{S}}

\title{\LARGE \bf CBF-RL: Safety Filtering Reinforcement Learning in Training with Control Barrier Functions}
\author{*Lizhi Yang, *Blake Werner, Massimiliano de Sa, Aaron D. Ames \thanks{* denotes equal contribution. All authors affiliated with Caltech MCE.\newline This research is supported in part by the Technology Innovation Institute (TII), BP p.l.c., and by The Dow Chemical Company project \#227027AW.}}
\date{September 2025}

\begin{document}

\maketitle
\begin{abstract}
    Reinforcement learning (RL), while powerful and expressive, can often prioritize performance at the expense of safety.  
    Yet safety violations can lead to catastrophic outcomes in real-world deployments.  
    Control Barrier Functions (CBFs) offer a principled method to enforce dynamic safety—traditionally deployed \emph{online} via safety filters.  While the result is safe behavior, the fact that the RL policy does not have knowledge of the CBF can lead to conservative behaviors. 
    This paper proposes CBF-RL, a framework for generating safe behaviors with RL by enforcing CBFs \emph{in training}. 
    CBF-RL has two key attributes: (1) minimally modifying a nominal RL policy to encode safety constraints via a CBF term, and (2) safety filtering of the policy rollouts in training.  Theoretically, we prove that continuous-time safety filters can be deployed via closed-form expressions on discrete-time roll-outs. Practically, we demonstrate that CBF-RL internalizes the safety constraints in the learned policy—both enforcing safer actions and biasing towards safer rewards—enabling safe deployment without the need for an online safety filter. 
    We validate our framework through ablation studies on navigation tasks and on the Unitree G1 humanoid robot, where CBF-RL enables safer exploration, faster convergence, and robust performance under uncertainty, enabling the humanoid robot to avoid obstacles and climb stairs safely in real-world settings without a runtime safety filter.
\end{abstract}
\section{Introduction}

Humanoid robots are capable of interacting with environments designed for humans. 
However, the complex environment, high-dimensional robot dynamics, and noise of the sensors also make them highly vulnerable to unsafe control inputs. 
One unsafe action could lead to damage to both the robot and its surroundings, and thus ensuring safety is essential.
Meanwhile, reinforcement learning (RL) has emerged as a powerful tool for humanoid robots to achieve diverse skills, but focuses mostly on performance \cite{crowley2025optimizing,li2025reinforcement,peng2025gait} and expressiveness \cite{he2025asap,allshire2025visual, truong2025beyondmimic, su2025hitter}.
\begin{figure}[t]
  \centering
  \includegraphics[width=0.9\linewidth]{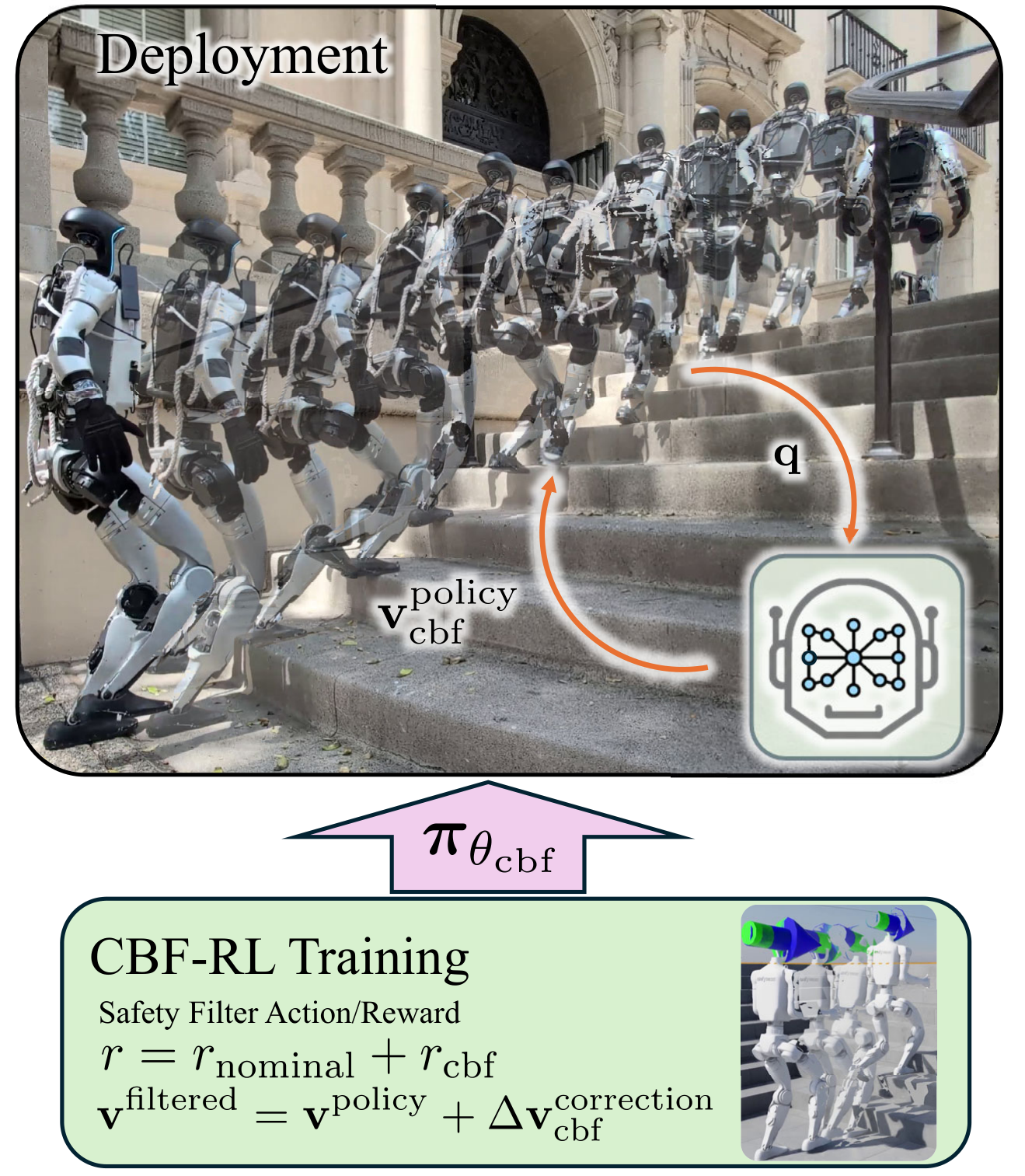}
  \caption{A humanoid robot trained to climb stairs with the CBF-RL framework. Safety is injected into training by both filtering the policy-proposed actions and also provide safety  rewards in addition to task and regularization rewards. During deployment the CBF policy retains safe behavior without a runtime filter.}
  \vspace{-25pt}
  \label{fig:dummy1}
\end{figure}
In this paper, we propose integrating formal safety mechanisms with the powerful exploration and exploitation abilities of RL so that learned policies can reduce or prevent catastrophic behaviors.
To achieve this, we turn to \emph{Control Barrier Functions (CBFs)} \cite{ames2016control} for a principled way to encode state-based safety constraints as forward-invariant sets. 
The CBF conditions are often enforced using \emph{safety filters} \cite{ames2019control}: quadratic programs satisfying the safety constraint by minimally modifying a proposed control input.

There are two key approaches to instantiating safety filters in RL. 
The first approach is \emph{safety filtering} the RL-proposed action and projects it into the safe set before execution \cite{cheng_end--end_2019, ma_model-based_2021,van_wijk_safe_2024, bejarano_safety_2025} or performing constrained updates of the gradient \cite{alshiekh_safe_2017, zhang_model-based_2021}.
This guarantees safety at runtime, but the filter must remain in the loop at deployment, and the learned policy may never internalize the constraint.
This prevents a high-dimensional agent, like a humanoid robot, from discovering novel or efficient behaviors since the exploration space is pruned too aggressively. 
Both also require solving an optimization program at every control step, which may be computationally expensive.
The second approach is \emph{reward shaping} where a residual augments the reward term to penalize states that approach or violate constraint boundaries \cite{krasowski2022provably, dunlap2023run,wabersich2021predictive,wang2022ensuring,nilaksh_barrier_2024, wang2025omni} and encourages the agent towards safer behaviors without active filtering. 
This alone does not directly enforce safe actions during training and is often sensitive to the choice of penalty weights, possibly being insufficient in safety-critical applications.  This paper proposes a fusion of these two approaches to integrating CBFs into RL. 

\vspace{-4pt}
\subsection{Contributions}
In this paper, we show that safety filtering and reward shaping are complementary, proposing \emph{CBF-RL}: a \emph{dual approach} that applies both a closed-form CBF-based safety filter and a barrier-inspired reward term during training.  This safety filters a nominal RL policy, enabling it to learn safe behaviors. 
Catastrophic unsafe actions are prevented by the active filter, and the reward term biases the policy toward avoiding safety interventions.
Therefore, the policy has direct corrective supervision--it observes what it would have done, how the filter corrects it, and how the reward changes. It also learns to propose actions that directly satisfy the barrier condition.
This enables the policy to show safe behaviors at deployment time without an active filter.

We evaluate CBF-RL, and its dual approach to safe RL, with ablations using a 2D navigation task with dynamics randomness, analyzing task completion rates and robustness tests.
We also train humanoid locomotion policies in IsaacLab \cite{mittal2023orbit} with full-order dynamics and domain randomization for obstacle avoidance and stair climbing tasks.  The approach is validated on hardware using a Unitree G1 humanoid with zero-shot sim-to-real policies to exhibit the effectiveness of the proposed framework in high-dimensional, complex systems.
With the dual-trained policy, the robot can successfully navigate around obstacles and climb stairs under command sequences that would lead to failure with nominal policies that don't leverage the CBF-RL framework.

Our contributions are as follows: 
\begin{itemize}
    \item \emph{Conceptually:} We propose a dual CBF-RL training framework that uses both CBF-based active filtering and barrier-inspired rewards during training, and can be deployed without a filter.
    \item \emph{Theoretically:} We provide a continuous-to-discrete-relationship analysis of CBF-RL and a closed-form solution for light-weight integration.
    \item \emph{Practically:} We empirically demonstrate across simulated, and hardware experiments that policies trained using the dual approach can internalize safety and reduce unsafe actions at deployment.
\end{itemize}

\vspace{-3pt}
\subsection{Related Work}
Due to the ease of modifying rewards for training, a number of works incorporate safety value functions into the reward structure to guide policies toward safety.
~\cite{krasowski2022provably, dunlap2023run} use a fixed penalty, 
~\cite{wabersich2021predictive,wang2022ensuring, wang2025omni} utilize correction-proportional penalties, and
~\cite{nilaksh_barrier_2024} proposes an explicit barrier-inspired reward shaping mechanism to reduce unsafe exploration.
These methods encourage safe behavior but do not explicitly direct the policy to exhibit safe behaviors during learning, instead solely relying on the policy to discover safer actions on its own, leading to slower training.

To address this, runtime CBF-based safety filters during training minimally modify the actions of the policy such that the system stays in the user-defined forward-invariant safe set typically by solving an optimization program, be it discrete-time due to the nature of RL training \cite{cheng_end--end_2019, ma_model-based_2021, zhang_model-based_2021, van_wijk_safe_2024, zhang_control_2025,bejarano_safety_2025}, or continuous-time \cite{hailemichael_safe_2022, emam_safe_2022,cheng_safe_2023} which empirically shows improved safety.
For humanoid locomotion, these methods are not ideal as humanoid robots have tight real-time and computation power constraints and inaccurate state estimation from sensor noise.
We instead filter only in simulation during training, and show that the resulting policy retains safety even without a runtime filter. 
We further provide theoretical verification that under certain conditions, continuous-time CBF can be used as conditions for forward invariance for RL simulations that are discrete in nature, and thus we can use the closed-form solution to the CBF-QP to accelerate each step in training.

Many papers utilize model-based approaches \cite{ma_model-based_2021,zhang_model-based_2021,hailemichael_safe_2022, du_reinforcement_2023}. 
which require access to accurate dynamics models.
While theoretically appealing, they are less practical for high-dimensional humanoids where dynamics are complex and uncertain.
Some works also do constrained updates of the gradient to ensure that the policy remains safe \cite{alshiekh_safe_2017, zhang_model-based_2021}; however, they also require solving optimization programs at each gradient update step and limit the exploration of the policy.
Our framework is model-free, requiring only derivatives of the reduced-order model 
(e.g. for the kinematics of a humanoid robot, its Jacobian $J$), and emphasizes lightweight integration with standard policy-gradient RL, in our case proximal policy optimization(PPO) \cite{schulman2017proximal}.
This also leads to the benefit of the policy being able to venture closer to the constraint boundaries, ensuring rich exploration.

In response to the issue of stochastic systems, robust extensions address uncertainty in dynamics or sensing through disturbance observers, Gaussian Process models, or robustified CBFs \cite{li_robust_2020,emam_safe_2022, hailemichael_safe_2022, cheng_safe_2023}.
These methods improve reliability under uncertainty but add significant complexity and computational burden. 
Here we show that by relying on domain randomization during training, dual-trained policies remain safer under uncertainty of the system without explicit models.  

Another line of work learns barrier-like certificates or relates value functions to barrier properties \cite{cohen_safe_2023,tan_value_2023}. 
These methods aim to automate the design of barrier functions or embed them in differentiable layers. 
Our approach assumes analytic barrier functions and focuses on pragmatic integration with RL training rather than barrier discovery.  

The above works have been applied to domains such as spacecraft inspection \cite{van_wijk_safe_2024}, drone control \cite{bejarano_safety_2025}, autonomous driving \cite{zhang_control_2025}, and driver assistance \cite{hailemichael_safe_2022}.
Much of the prior work captures part of the safety challenge, but none directly combines filtering and shaping in a way that allows a policy to \emph{internalize safety} during training and then act autonomously without a filter at deployment, especially not on a high-dimensional system such as a humanoid robot.
This distinction defines the novelty of our contribution.

\section{Background}

\newsec{Reinforcement Learning}
We consider the standard RL formulation of a Markov decision process (MDP) $(\mathcal{X}, \mathcal{U}, P, r, \gamma)$. 
At each timestep $k$, an agent selects an action $\mb{u_k} \in \mathcal{U}$, according to a policy $\mb{\pi_\theta}(\mb{u_k}|\mb{x_k})$ based on an observed state $\mb{x_k} \in \mathcal{X}$, and receives a reward $r(\mb{x_k}, \mb{u_k})$. 
The environment follows the transition dynamics $\mb{x_{k+1}} \sim P(\cdot|\mb{x_k},\mb{u_k})$. 
The goal is then to maximize the expected discounted return $\mathbb{E}\!\left[ \sum_{k=0}^{\infty} \gamma^t r(\mb{x_k}, \mb{u_k}) \right]$.
During deployment, actions are selected as the conditional expectation $\mb{u_k} = E[\mb{\pi_{\theta}}(\mb{u_k} | \mb{x_k})]$.
In this work, we specifically use model-free policy-gradient actor-critic methods such as PPO \cite{schulman2017proximal}, though our approach is agnostic to the specific RL algorithm.
RL also often suffers from reward sparsity when unsafe events like obstacle collisions are rare.
Thus the policy seldom experiences consequences of unsafe actions, leading to vanishing gradients and unstable, slow training.
As such, one part of our method also does reward densification to mitigate this by designing $r(\mb{x_k}, \mb{u_k})$  to give informative nonterminal signals related to safety.
\begin{figure*}
    \centering
    \includegraphics[width=0.95\linewidth]{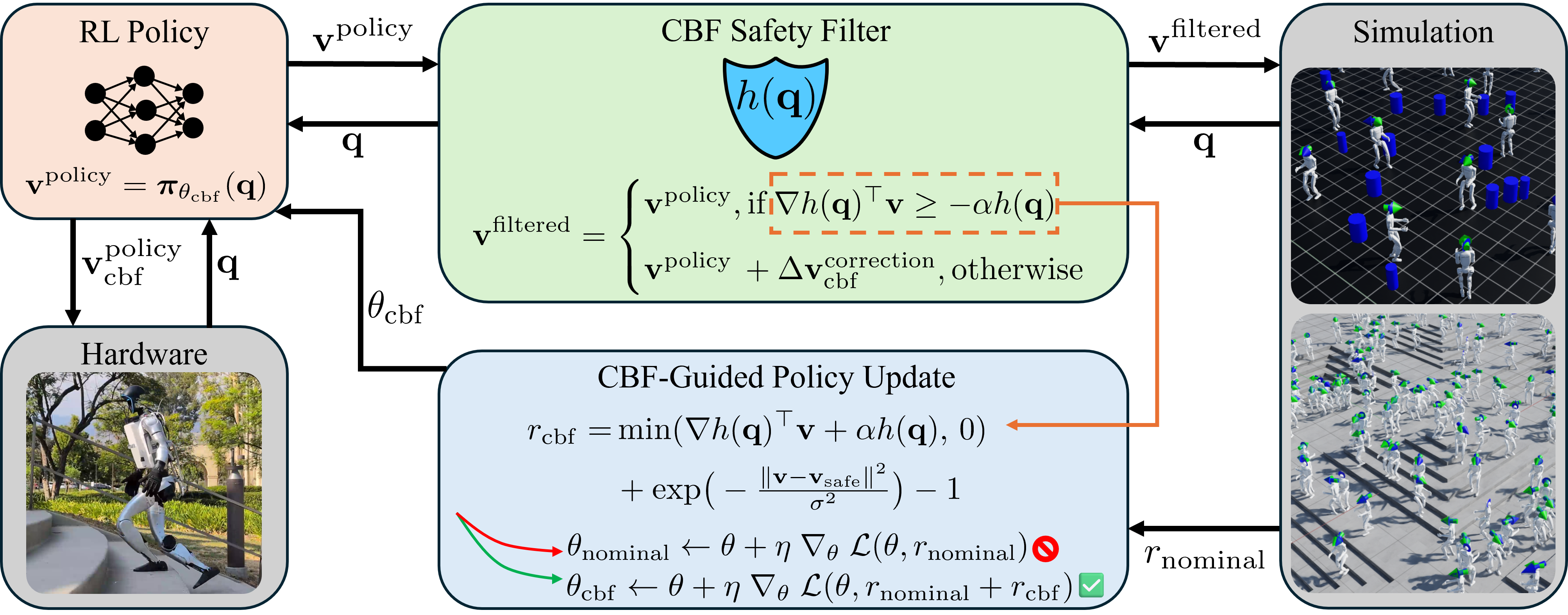}
    \caption{CBF-RL framework: For one given task, the user defines the safety barrier function $h(\mathbf{q})$ and the accompanying $\nabla h(\mb{q})$. 
    During training, the RL policy proposes action $\mb{v}^{\mathrm{policy}}$;
    the CBF safety filter then calculates the closed-form solution of the CBF-QP $\mb{v}^{\mathrm{filtered}}$ and the safety reward $r_{\mathrm{cbf}}$ based on proposed action $\mb{v}^{\mathrm{policy}}$ and agent configuration $\mb{q}$.
    The RL agent executes $\mb{v}^{\mathrm{filtered}}$ in the massively-parallel discretized environment, and the policy is updated with the combination of task, regularization, and safety rewards $r = r_\mathrm{nominal} + r_\mathrm{cbf}$.
    During deployment, the policy is able to output safe actions $\mathbf{v}^{\mathrm{policy}}_{\mathrm{cbf}}$ without needing an explicit runtime filter. Sample code for the single integrator example with CBF-RL reward core construction can be found at \url{https://github.com/lzyang2000/cbf-rl-navigation-demo}.
    }
    \vspace{-18pt}\label{fig:system}
\end{figure*}

\newsec{Reduced-Order Models}
Consider a continuous-time system with dynamics
$
    \dot{\mb{x}} = \mb{\phi}(\mb{x}, \mb{u}),
$
where $\mb{x} \in \R^n$ is the state and $\mb{u} \in \R^m$ is the control input. In our case the state $\mb{x}$ may be very high-dimensional, e.g. joint positions and velocities of a robot. Thus we consider a reduced-order state $\mb{q} \in \R^{n_q}$, $n_q < n$, representing key lower-dimensional features such as the robot’s center of mass position. We define a projection $\mb{p}: \R^n \to \R^{n_q}$ that projects the full-order state onto the reduced-order state. Given a locally Lipschitz continuous feedback control law $\mb{v} = \mb{k}(\mb{q})$, the reduced-order state $\mb q$ is
\begin{align}
    \dot{\mb{q}} = \frac{\partial \mb{p}}{\partial \mb{x}} \, \mb{\phi}\big(\mb{x}, \, \mb{\psi}(\mb{x}, \mb{v})\big) \nonumber &\approx \mb{f}(\mb{q}) + \mb{g}(\mb{q}) \mb{v} 
    \\
     &= \mb{f}(\mb{q}) + \mb{g}(\mb{q}) \mb{k}(\mb{q}),
    \label{eq:reduced_order_sys}
\end{align}
where $\mb{\psi}(\mb{x},\mb{v})$ is a control interface lifting the reduced-order input $\mb{v}$ to a full-order input $\mb{u}$. See \cite{cohen_safe_2023,cohen2025safety} for the connections between reduced and full order models.
\newsec{Control Barrier Function Safety Filters}
In the control barrier function framework, a set of ``safe states" for the system, $\Sc \subseteq \R^{n_q}$, is encoded as the zero superlevel set of a continuously differentiable function $h: \R^{n_q} \to \R$,
\begin{align} 
\label{eq: Safe Set S}
    \Sc \coloneqq \left\{\mb{q}\in\R^{n_q} \,\big|\, h(\mb{q}) \geq 0 \right\},\\
    \partial\Sc \coloneqq \left\{\mb{q}\in\R^{n_q} \,\big|\, h(\mb{q}) = 0 \right\},
    \\
    \mathrm{int}\left(\Sc\right) \coloneqq \left\{\mb{q}\in\R^{n_q} \,\big|\, h(\mb{q}) > 0 \right\}.
\end{align} 
The aim of safety-critical control is to design a feedback control law $\mb{k}(\mb{q})$ that renders $\Sc$ forward invariant.

\begin{definition}

    A set $\Sc \subseteq \R^{n_q}$ is \textit{forward invariant} for \eqref{eq:reduced_order_sys} if, for every initial state $\mb{q}(t_0) \in \Sc$, the resulting state trajectory $\mb{q}: I \subseteq \R \to \R^{n_q}$ remains in $\Sc$ for all $t \in I \cap \R_{\geq t_0}$.
    
\end{definition}

For control-affine systems as in \eqref{eq:reduced_order_sys}, the forward invariance of $\Sc$ can be enforced using CBFs \cite{ames2016control,ames2019control}.

\begin{definition}
    Let $\Sc \subset \R^{n_q}$ be a set as in \eqref{eq: Safe Set S}, with $h$ satisfying $\nabla h|_{\partial \Sc} \neq 0$. Then, $h$ is a \textit{control barrier function} on $\R^{n_q}$ if there exists\footnotemark a $\gamma \in \mathcal{K}_{\infty}^e$ such that for all $\mb{q} \in \R^{n_q}$,
    \begin{align} 
    \label{eq: CBF}
    \!\!\!\! \sup_{\mb{v} \in \R^m} \left\{\dot{h}(\mb{q},\mb{v}) = L_{\mb{f}} h(\mb{q}) + {L_{\mb{g}} h(\mb{q})} \mb{v}\right\} > -\gamma(h(\mb{q})) .
    \end{align} 
\end{definition}
\footnotetext{1: A function $\alpha: \R \to \R$ belongs to $\mathcal K_\infty^e$ if it is increasing, continuous, and satisfies $\alpha(0) = 0, \; \lim_{r \to \pm \infty} \alpha(r) = \pm \infty$.}

Given a CBF $h$ and function $\gamma \in \mathcal K_\infty^e$, the set of control inputs satisfying \eqref{eq: CBF} at $\mb q$ is given by
\begin{align} 
    \mathcal{U}_\mathrm{CBF}(\mb{q}) = \left \{\mb{v} \in \R^m \, \big | \,  \dot{h}(\mb{q}, \mb{v})  \geq -\gamma(h(\mb{q})) \right\}.
\end{align} 
Any locally Lipschitz controller $\mb k$ for \eqref{eq:reduced_order_sys} for which $\mb{k}(\mb q) \in\mathcal{U}_\mathrm{CBF}(\mb{q}) \; \forall \mb q \in \Sc$ enforces forward invariance of $\Sc$ \cite{ames2016control}.

For real-world robotic systems with zero-order hold, sampled-data implementations, it is generally impossible to choose continuous control actions that create the closed-loop system \eqref{eq:reduced_order_sys}. As such, for the remainder of this work, we focus on the discrete-time analogues of these systems,
\begin{align} \label{eq: Discrete Dynamics}
    \mb{q}_{k+1} & = \mb{F}(\mb{q}_{k}) + \mb{G}(\mb{q}_{k}) \mb{v}_k,  & \forall k \in \mathbb{Z}_{\geq 0}, 
\end{align} 
where $\mb{F}: \R^{n_q} \to \R^{n_q}$ and $\mb{G}: \R^{n_q} \to \R^{n_q \times m}$ are the discretization of \eqref{eq:reduced_order_sys} over a time interval $\mathrm{\Delta t} >0$ for a constant input $\mb{v}$. Here, we focus on enforcing forward invariance at sample times\footnote{\vspace{-1em}2: See \cite{breeden2021control} for a discussion of zero-order-hold, intersample safety.} $k \mathrm{\Delta t}$. This can be achieved using a discrete-time CBF \cite{agrawal2017discrete,ahmadi2019safe}.  
\begin{definition}
    Let $\Sc \subseteq \R^{n_q}$ be as in \eqref{eq: Safe Set S} and $\rho \in [0,1]$. The function $h$ is a \textit{discrete-time control barrier function (DTCBF)} for \eqref{eq: Discrete Dynamics} if $\forall \mb{q}\in \Sc$ there exists a $\mb{v} \in \R^m $ for which
    \begin{align} 
        h(\mb{F}(\mb{q}) + \mb{G}(\mb{q}) \mb{v}) \geq \rho h(\mb{q}).
    \end{align} 
\end{definition}

Similar to CBFs and continuous-time systems, DTCBFs keep the discrete-time system \eqref{eq: Discrete Dynamics} safe for each $k \in \mathbb Z_{\geq 0}$. Here the value of $h$ is lower bounded by a geometrically decaying curve, $h(\mb{q}_k) \geq \rho^k h(\mb{q}_0)$ \cite{agrawal2017discrete}.
Thus, they can be used to generate safe control actions through an optimization program wherein a desired but potentially unsafe input $\mb{v}^\textup{des}_k \in \R^m $ is minimally modified to produce a safe input: 
\begin{align}
    \mb{v}^\textup{safe}_k &= \underset{\mb{v}_k\in \R^m }{\textup{arg min}}  \quad \Vert \mb{v}_k - \mb{v}^\textup{des}_k(\mb{q_k})\Vert^2\\
   & \textup{s.t.}  \quad h(\mb{F}(\mb{q}_{k}) + \mb{G}(\mb{q}_{k}) \mb{v}_k) \geq \rho h(\mb{q_k}).
\end{align} 

\section{Dual Approach to CBF-RL}
In order to apply CBFs to RL pipelines, we characterize the relationship between continuous-time CBFs and discrete updates of the RL environment. With this relationship, we can utilize the closed-form solution of the continuous-time CBF-QP to understand its effect on the RL environment.
\begin{lemma}\label{lem:barrier-discretization}
    Suppose $h: \R^{n_q} \to \R$ is a $C^1$ CBF for the continuous-time single integrator $\dot q = v$, where $ q, v \in \R^{n_q}$. Let $k_s: \R^{n_q} \to \R^{n_q}$ be any safe, locally Lipschitz controller for the continuous-time integrator, satisfying
    \begin{align}
        \nabla h(\mb q)^\top k_s(\mb q) \geq -\alpha h(\mb q), \; \forall \mb q \in \R^{n_q},
    \end{align}
    for some $\alpha > 0$. Let $\mb f_{\mathrm{\Delta t}}: \R^{n_q} \times \R^{n_q} \to \R^{n_q},  (\mb q, \mb v) \mapsto \mb q + \mathrm{\Delta t}\, \mb v$ be the Euler discretization of the single integrator with time step $\mathrm{\Delta t} > 0$. There exists a continuous function $R: \R^{n_q} \times \R^{n_q} \to \R$ such that $\forall \mb q \in \R^{n_q}$, $\lim_{\norm{\mb w} \to 0}\frac{R(\mb q, \mb w)}{\norm{\mb w}} = 0$ and
    \begin{align}
        h(\mb f_\mathrm{\Delta t}(\mb q, k_s(\mb q))) \geq (1-\mathrm{\Delta t}\,\alpha) h(\mb q) - |R(\mb q, \mathrm{\Delta t}\, k_s(\mb q))|.
    \end{align}
\end{lemma}
\begin{proof}
    By \cite[Chapter 5.2 (2)]{pugh2002real}, there exists a continuous function $R: \R^{n_q}\times \R^{n_q} \to \R$ satisfying $\lim_{\norm{\mb w} \to 0}\frac{R(\mb q, \mb w)}{\norm{\mb w}} = 0$ and $h(\mb q + \mb w) = h(\mb q) + \nabla h(\mb q)^\top \mb w + R(\mb q, \mb w) \; \forall \mb q, \mb w \in \R^{n_q}$. Taking $\mb w = \mathrm{\Delta t}\, k_s(\mb q)$, we calculate $ h(\mb q + \mathrm{\Delta t}\, k_s(\mb q))$ as 
    \begin{align}
        &= h(\mb q) + \mathrm{\Delta t}\, \cdot [\nabla h(\mb q)^\top k_s(\mb q)] + R(\mb q, \mathrm{\Delta t}\,k_s(\mb q))\\
        &\geq h(\mb q) - \mathrm{\Delta t} \cdot \alpha h(\mb q) - |R(\mb q, \mathrm{\Delta t}\,k_s(\mb q))|.
    \end{align}
    Combining $h$ terms, the result follows.
\end{proof}
Using Lemma \ref{lem:barrier-discretization}, we bound the evolution of the barrier function along trajectories of the Euler-discretized integrator.
\begin{theorem}[Continuous to Discrete Safety]
    Consider the setting of Lemma \ref{lem:barrier-discretization}. Suppose there exists a compact, forward-invariant set $K \subseteq \R^{n_q}$ for the discrete-time integrator dynamics $\mb q_{k+1} = \mb f_{\mathrm{\Delta t}}(\mb q_k, k_s(\mb q_k))$. Then, $\forall \mathrm{\Delta t} > 0$, $\mu(\mathrm{\Delta t}) = \sup_{\mb q \in K} |R (\mb q, \mathrm{\Delta t}\, k_s(\mb q))| < + \infty$ and $\lim_{\mathrm{\Delta t} \to 0} \mu(\mathrm{\Delta t})/\mathrm{\Delta t} = 0$. Further, provided $(1 - \mathrm{\Delta t}\, \alpha) \in [0, 1)$, for $ \mb {q}_0 \in K$, 
\begin{align}
        h(\mb q_k) \geq (1 - \mathrm{\Delta t}\, \alpha)^k h(\mb q_0) - \frac{\mu(\mathrm{\Delta t})}{\mathrm{\Delta t}\, \alpha}, \; \forall k \geq 0.
    \end{align}
\end{theorem}
\begin{remark}
    As a consequence of the limiting behavior of $\mu$, as we take the discretization step $\mathrm{\Delta t}$ to zero, the standard DTCBF bound, $h(\mb q_k) \geq (1 - \mathrm{\Delta t}\, \alpha)^k h(\mb q_0)$, is recovered.
\end{remark}
\begin{proof}
    Fix $\mathrm{\Delta t}> 0$. Since $R$ and $k_s$ are continuous, compactness of $K$ implies $\mu(\mathrm{\Delta t})$ is finite $ \forall \mathrm{\Delta t} > 0$. To establish the limit $\lim_{\mathrm{\Delta t} \to 0} \mu(\mathrm{\Delta t})/\mathrm{\Delta t} = 0$, we note that $ \lim_{\mathrm{\Delta t} \to 0} \tfrac{|R(\mb q, \mathrm{\Delta t}\, k_s(\mb q))|}{\mathrm{\Delta t}}= 0 \; \forall \mb q \in K$ implies
    \begin{align}
        \sup_{\mb q \in K} \lim_{\mathrm{\Delta t} \to 0} \tfrac{|R(\mb q, \mathrm{\Delta t}\, k_s(\mb q))|}{\mathrm{\Delta t}} = 0.
    \end{align}
    Compactness of $K$ and joint continuity of $|R(\mb q, \mathrm{\Delta t}\, k_s(\mb q))|$ in $\mb q$ and $\mathrm{\Delta t}$ implies uniform continuity of $|R(\mb q, \mathrm{\Delta t}\, k_s(\mb q))|$; this lets us interchange limit and supremum. We conclude $\lim_{\mathrm{\Delta t} \to 0} \sup_{\mb q \in K} \frac{|R(\mb q, \mathrm{\Delta t}\, k_s(\mb q))|}{\mathrm{\Delta t}} = 0 \Rightarrow \lim_{\mathrm{\Delta t} \to 0} \tfrac{\mu(\mathrm{\Delta t})}{\mathrm{\Delta t}} = 0$.

    Now, we establish the bound using a comparison system. If $\mb y_{k+1} = \rho \mb y_k - |\mb d_k|$ for all $ k \geq 0$ and $|\mb d_k| \leq \mu$, a geometric series argument establishes $\mb y_k \geq \rho^k \mb y_0 - (\frac{1}{1 - \rho}) \mu \; \forall k \geq 0$. Fix $\mb q_0 \in K$. Since $K$ is forward invariant for the closed-loop system, $|R(\mb q_k, \mathrm{\Delta t}\, k_s(\mb q_k))| \leq \mu(\mathrm{\Delta t})$ $\forall k \geq 0$. Using Lemma \ref{lem:barrier-discretization}, we take $\rho = 1 - \mathrm{\Delta t} \alpha$, $\mu(\mathrm{\Delta t}) = \sup_{\mb q \in K} |R (\mb q, \mathrm{\Delta t}\, k_s(\mb q))|$, and conclude by comparison with $\mb y_k$ that the bound
    \begin{align}
        h(\mb q_k) \geq (1 - \mathrm{\Delta t}\, \alpha)^k h(\mb q_0) - \frac{\mu(\mathrm{\Delta t})}{\mathrm{\Delta t} \,\alpha},
    \end{align}
    does indeed hold for all $\mb q_0 \in K$.
\end{proof}
\vspace{-5pt}
The above means that with small enough $\Delta t$ (typically $\Delta t \le 0.01$s, governed by the stability limits and solve time of the physics engine), continuous-time CBF tools can be directly applied to discrete-time RL environments. We provide the two core parts of CBF-RL shown in Fig. \ref{fig:system}:

\newsec{Safety-filtering during training}
At training time, the policy would generate desired actions that are not necessarily safe $\mb{v}_k^{\mathrm{policy}}$ at step $k$.
A safety filter is then applied to enforce safe behaviors and guide the system to 'learn' the safety filter as part of the natural closed-loop dynamics.
As we have proven the relationship between the continuous-time and discrete-time CBF conditions, we replace the typically nonlinear DTCBF constraint with the continuous-time CBF first-order inequality. This reduces the safety filtering problem to a single linear-constraint quadratic program (QP) \cite{ames2019control}:
\vspace{-5pt}
\begin{align} 
    \mb{v}_k^{\mathrm{safe}} &= \arg \min_{\mb{v}_k\in \mathbb{R}^{n_q}} \frac{1}{2} ||\mb{v}_k-\mb{v}_k^{\mathrm{policy}}||^2 \\
    &\mathrm{s.t.} \quad  \nabla h(\mb{q}_k)^\top \mb{v}_k \geq -\alpha h(\mb{q}_k).
\end{align} 
\vspace{-2em}

While this QP provides a safe control input, solving an optimization program numerically at every step of RL training is computationally undesirable, especially in massively parallel environments such as IsaacLab \cite{mittal2023orbit}. Fortunately, since the QP has a single linear constraint, it can be solved analytically in closed-form:
\begin{align}
\mb{v}_k^{\mathrm{safe}}
&=
\begin{cases}
 \mb{v}_k^{\mathrm{policy}}, & \text{if } \mathbf{a}_k^\top \mb{v}_k^{\mathrm{policy}} \ge b_k \\[6pt]
 \mb{v}_k^{\mathrm{policy}} + \dfrac{\bigl(b_k - \mathbf{a}_k^\top \mb{v}_k^{\mathrm{policy}}\bigr)\,\mathbf{a}_k}{\|\mathbf{a}_k\|^2}, & \text{ o.w.}
\end{cases}
\label{eq:closed-form-sol}
\\[4pt]
\mathbf{a}_k &:= \nabla h(\mb{q}_k), \qquad
b_k := -\alpha\,h(\mb{q}_k).
\end{align}

\vspace{-1em}
\newsec{Penalizing unsafe behavior}
In addition to filtering the policy actions at training time, we also quantify how safe the environment is to inform training. 
To this end, we modify the rewards to include $r_{{\mathrm{CBF}}}$ defined as
\vspace{-4pt}
\begin{align} 
    r_{{\mathrm{cbf}}}(\mb{q}_k, \mb{v}_k) = &\min\!\left( \mb{a}_k^\top\mb{v}^{\mathrm{policy}}_k - b_k,\, 0\right) \\ &+ \big (\exp\!\left(-\tfrac{\|\mb{v}^{\mathrm{policy}} - \mb{v}^{\mathrm{safe}}\|^2}{\sigma^2}\right) -1\big)
    \label{eq:reward-cbf-penalty}
\end{align} 
and the whole reward is thus
$
    r = r_{{\mathrm{nominal}}} + r_{{\mathrm{cbf}}}.
$

Intuitively, this reward penalizes actions whenever the safety filter is activated, and also incentivizes the model to take actions as close to the safe actions as possible to reduce the intervention of the filter. Because the penalty term $\exp(\dots)-1$ is strictly lower-bounded by $-1$, it provides a smooth learning signal without causing unbounded instability with respect to the primary task rewards.
The whole algorithm is as shown in
Alg. \ref{Alg:CBF_RL}.
\vspace{-2ex}
\begin{center}
\resizebox{0.9\columnwidth}{!}{
\begin{minipage}{\columnwidth}
\begin{algorithm}[H]
\caption{RL Training with Discrete-Time CBF Safety }
\label{Alg:CBF_RL}
\begin{algorithmic}[1]
\State Initialize policy parameters $\theta$, initial configuration $q_0$, safety function $h$
\For{step $=1$ to $N_{\mathrm{steps}}$}
\State Initialize $q_0$, observation $\mb{o}_0$
\For{$k = 0$ to $T-1$}
\State $\mb{v}^{\mathrm{policy}}_k \gets \pi_\theta(\mb{o}_k)$
\State $\mb{q}^{\mathrm{policy}}_{k+1} \gets \mb{q}_k + \Delta t \, \mb{v}^{\mathrm{policy}}_k$
\State $\mb{a}_k = \nabla h(\mb{q}_k)$, $b_k = -\alpha h(\mb{q}_k)$
\State Compute CBF condition: $c = \mb{a}_k^\top\mb{v}_k^{\mathrm{policy}} - b_k$
\State \textbf{if} {$c \ge 0$} \textbf{then} $\mb{v}^{\mathrm{safe}}_k \gets \mb{v}^{\mathrm{policy}}_k$, 
\State \textbf{else} $\mb{v}^{\mathrm{safe}}_k \gets \mb{v}^{\mathrm{policy}}_k + \dfrac{(b_k - \mathbf{a}_k^\top v^{\mathrm{policy}}_k)\,\mathbf{a}_k^{\!\top}}{\|\mathbf{a}_k\|^2}$
\State $\mb{q}^{\mathrm{safe}}_{k+1} \gets \mb{q}_k + \Delta t\,\mb{v}^{\mathrm{safe}}_k$
\State $\mb{q}^{\mathrm{env}}_{k+1}, \mb{o}_{k+1} \gets \textsc{environment\_update}(\mb{q}^{\mathrm{safe}}_{k+1})$
\State $r \gets w\cdot [\min\!\left( \mb{a}_k^\top\mb{v}^{\mathrm{policy}}_k - b_k,0\right) + \exp\!\left(-\tfrac{\|\mb{v}^{\mathrm{policy}} - \mb{v}^{\mathrm{safe}}\|^2}{\sigma^2}\right)-1] +  R(\mb{q}^{\mathrm{env}}_{k+1})$
\State Store transition: $(\mb{q}_k,\mb{o}_k,\mb{q}^{\mathrm{policy}}_{k+1},\mb{q}^{\mathrm{safe}}_{k+1},\mb{q}^{\mathrm{env}}_{k+1}, r)$
\EndFor
\State Update policy parameters $\theta \gets \eta \nabla_\theta \mathcal{L}(\theta, r)$
\EndFor
\end{algorithmic}
\end{algorithm}
\end{minipage}
}
\end{center}
\newsec{Single Integrator Example}
To demonstrate our proposed approach and analyze the effect of each component of CBF-RL, we perform extensive ablation studies on a single integrator navigation task.
The agent is placed into a 2D world with obstacles and tasked with going to a specified goal.
The starting, obstacle, and goal locations are all randomized, and extra care is taken such that the agent always initializes in a safe state following \cite{bejarano_safety_2025}.
The single integrator has dynamics $\mb{q}_{k+1}= \mb{q}_k + \mb{v}_k \Delta t$
where $\mb{q}=(x,y)$ is the robot position, $\mb{v}_k$ is the velocity of the agent, and $\Delta t$ is the timestep size.
The safety barrier function $h$ is defined as 
\begin{align}
\vspace{-2em}
h(\mb{q}) &=
\begin{aligned}[t]
&\min \Big\{ \min_{j}\big(\|\mb{q} - \mb{p}_j\| - (r_{\text{agent}} + r_j)\big), \\
&\quad x - r_{\text{agent}},\ (L - x) - r_{\text{agent}}, \\
&\quad y - r_{\text{agent}},\ (L - y) - r_{\text{agent}} \Big\}
\end{aligned} 
\label{eq:cbf_h} \\
\nabla h(\mb{q}) &=
\begin{cases}
\dfrac{\mb{q} - \mb{p}_{j^\star}}{\|\mb{q} - \mb{p}_{j^\star}\|}, & j^\star \in \arg\min_j h_j(\mb{q}), \\[1ex]
\pm e_x, & \text{if left/right wall active}, \\
\pm e_y, & \text{if bottom/top wall active},
\end{cases}
\label{eq:cbf_grad}
\end{align}
where $\mb{p}_j$ is the position of the jth obstacle with radius $r_j$, the agent also has radius $r_{\text{agent}}$ and the size of the world is $L$.
Thus we can formulate the training-time safety filter with \eqref{eq:closed-form-sol} and define our reward modification associated with the safety with \eqref{eq:reward-cbf-penalty}.
Full reward terms are shown in Table II.

\newsec{Ablation} 
To validate our method, we train 4 variants (Dual, Reward-only, Filter-only, Nominal) and test 12 variants following Table III for 1500 steps with 4096 parallel environments. 
We also evaluate policies trained with filtered action then deployed without a runtime filter (rt. filt.).
The training progress can be seen in Fig.~\ref{fig:toy_rewards}, where we observe that the Dual and Filter-only approaches achieve rapid convergence while remaining safe throughout training.
Furthermore, as illustrated by the trajectory comparisons in Fig.~\ref{fig:toy_results} as an example from over 1000 random test environments, the Dual approach is able to reach the goal in both domain randomized and non-domain randomized settings, even without a runtime filter, while the other methods fail to do so.
Notably, the Filter Only approach performs well only with an active safety filter and degrades markedly without it.

\newsec{Robustness}
To further investigate the robustness of the policy induced by domain randomization (DR), we train the dual approach with noise on the dynamics model, i.e.
$
    \mb{q}_{k+1}= \mb{q}_k + (\mb{v}_k + \mb{d})\Delta t
$
where $\mb{d}$ follows the standard normal distribution scaled by 20\% to the maximum velocity.
It is observed that the policy trained with the dual method overall suffers least from the dynamics disturbance as can  be seen in Table \ref{table: ablation DR}.
\vspace{-1ex}
\begin{table}[h]
\centering
\renewcommand{\arraystretch}{1.2}
\setlength{\tabcolsep}{6pt}
\caption{Success rates over 1000 random test environments  for different methods, with and without DR. The dual policy consistently performs well and suffers less degradation due to dynamics uncertainty.} 

\begin{tabular*}{\columnwidth}{l c c c}
\hline
 & \,  \textbf{Dual}\,   & \,  \textbf{Dual (w/o rt. filt.)} \,  &\,    \textbf{Reward Only} \\
\hline
No DR & $99.0\%$ & $92.7\%$ & $91.9\%$ \\
DR    & $\mb{99.0\% (-0\%)}$ & $\mb{91.7\% (-1\%)}$ & $87.6\%(-4.3\%)$ \\
\hline
\end{tabular*}

\vspace{0.4em} 

\begin{tabular*}{\columnwidth}{l c c c}
\hline
 & \,  \textbf{Filter Only} \, & \textbf{Filter Only (w/o rt. filt.)} \, & \, \textbf{Nominal}  \,\\
\hline
No DR & $98.8\%$ & $38.7\%$ & $51.4\%$ \\
DR    & $96.7\%(-2.1\%)$ & $36.8\%(-1.9\%)$ & \hspace{-2em} $55.0\%(+3.6\%)$ \\
\hline
\end{tabular*}

\vspace{-10pt}
\label{table: ablation DR}
\end{table}
\begin{figure*}[t]
\centering

\begin{minipage}[t]{0.48\textwidth}
    \vspace{0pt}
    \centering
    \includegraphics[width=0.9\linewidth]{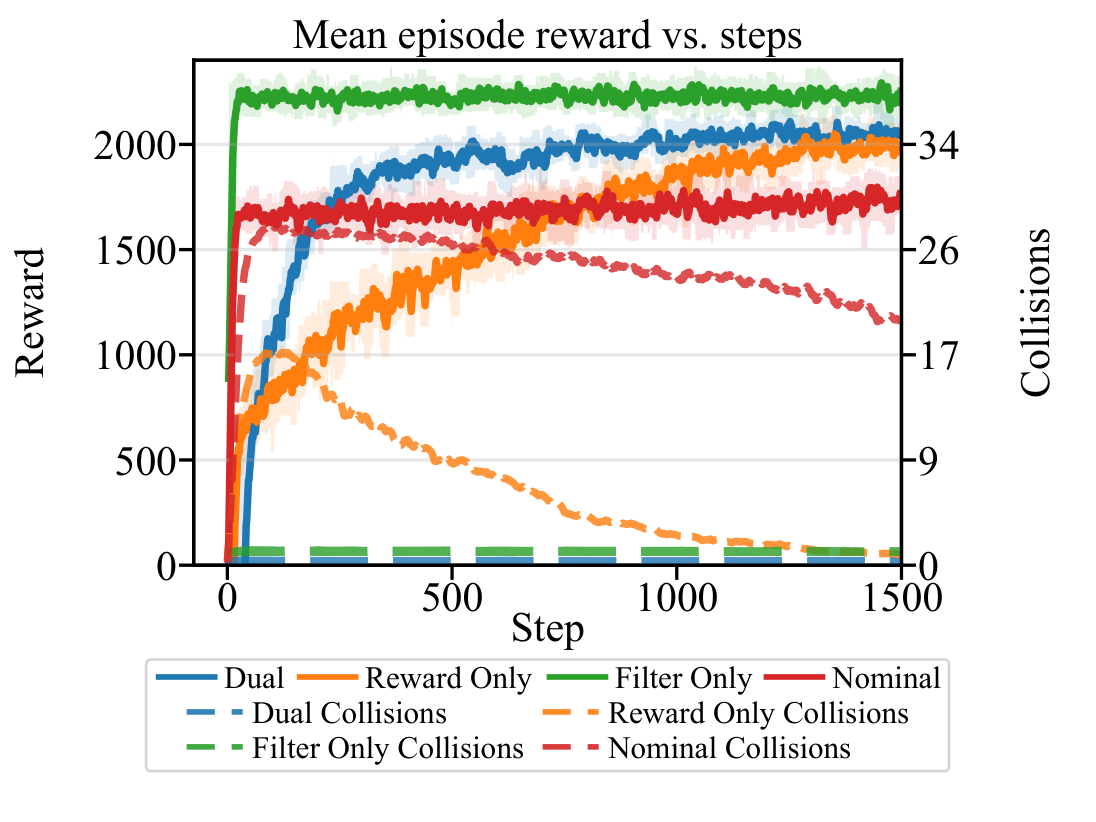}
    \vspace{-1ex}
    \caption{Training progress with the Dual, Reward Only, Filter Only and Nominal methods. Dual and Filter Only achieve faster convergence and avoid training-time safety violations.}
    \label{fig:toy_rewards}
\end{minipage}\hfill
\begin{minipage}[t]{0.48\textwidth}
      \vspace{0pt}
    \centering    
    \includegraphics[width=0.9\linewidth]{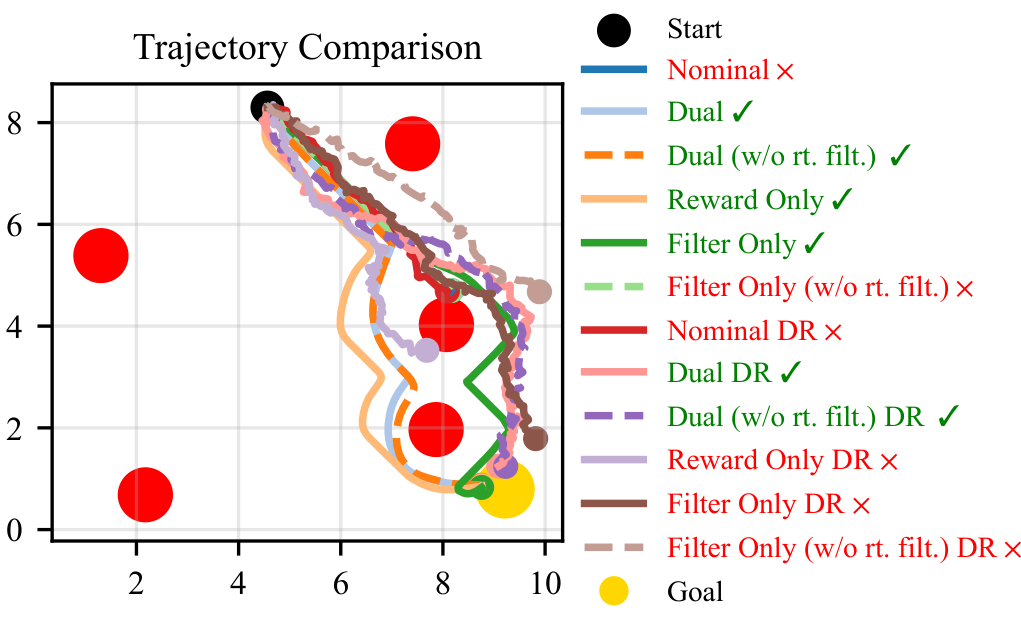}
    \caption{Trajectory comparisons of one example simulation. The black dot is the start and the yellow circle is the goal. Success = reaching the goal; failure = collision with obstacles or wall.}
    \label{fig:toy_results}
\end{minipage}
\vspace{1em} 

\begin{minipage}[t]{0.48\textwidth}
      \vspace{0pt}
    \centering
    \footnotesize
    \captionof{table}{Reward terms for single integrator navigation. 
    Agent is rewarded for staying alive and closing the distance to the goal and penalized for hitting the obstacles / walls, exceeding the set time to complete the task, and violating the CBF conditions or not proposing safer actions}
        \label{tab:reward_toy}

    \begin{tabular}{l l}
    \toprule
    \textbf{Reward} & \textbf{Formula} \\
    \midrule
    $r_{\mathrm{goal}}$ & $1.0 \cdot \mathbf{1}(\text{goal reached})$ \\
    $r_{\mathrm{obstacle}}$ & $-1.0 \cdot \mathbf{1}(\text{obstacle collision})$ \\
    $r_{\mathrm{wall}}$ & $-1.0 \cdot \mathbf{1}(\text{wall collision})$ \\
    $r_{\mathrm{progress}}$ & $20.0 \cdot \dfrac{\|\mathbf{p}_{t-1}-\mathbf{g}\|-\|\mathbf{p}_t-\mathbf{g}\|}{v_{\max}\Delta t}\cdot \mathbf{1}(\text{active})$ \\
    $r_{\mathrm{alive}}$ & $0.01 \cdot \mathbf{1}(\text{active})$ \\
    $r_{\mathrm{cbf}}$ &
    \(\begin{aligned}[t]
      &100\cdot\Big(\min(\nabla h(\mathbf{q})^\top \mathbf{v}+\alpha h(\mathbf{q}),\,0) \\
      &+ \exp\!\big(-\tfrac{\|\mathbf{v}_{\mathrm{policy}}-\mathbf{v}_{\mathrm{safe}}\|^2}{0.5^2}\big) -1\Big)\cdot\mathbf{1}(\text{active})
    \end{aligned}\) \\
    $r_{\mathrm{timeout}}$ & $-10.0 \cdot \mathbf{1}(\text{time exceeded})$ \\
    \bottomrule
    \end{tabular}
    
\vspace{-4em}
\end{minipage}\hfill
\begin{minipage}[t]{0.48\textwidth}
      \vspace{-3ex}
    \centering
    \footnotesize
        \captionof{table}{List of method configurations for single integrator navigation. We ablate all permutations of the two main components of CBF-RL.}

    \setlength{\tabcolsep}{3pt}
    \renewcommand{\arraystretch}{1.1}
    \begin{tabular}{l l l p{0.08\textwidth}}
    \toprule
    \textbf{Method} & \textbf{Training} & \textbf{Deployment} & \textbf{DR} \\
    \midrule
    Nominal & Nominal & No Runtime Filter & No \\
    Dual & Reward+Filter & Runtime Filter & No \\
    Dual (w/o rt. filt.) & Reward+Filter & No Runtime Filter & No \\
    Reward Only & Reward & No Runtime Filter & No \\
    Filter Only & Filter & Runtime Filter & No \\
    Filter Only (w/o rt. filt.) & Filter & No Runtime Filter & No \\
    Nominal DR & Nominal & No Runtime Filter & Yes \\
    Dual DR & Reward+Filter & Runtime Filter & Yes \\
    Dual (w/o rt. filt.) DR & Reward+Filter & No Runtime Filter & Yes \\
    Reward Only DR & Reward & No Runtime Filter & Yes \\
    Filter Only DR & Filter & Runtime Filter & Yes \\
    Filter Only (w/o rt. filt.) DR & Filter & No Runtime Filter & Yes \\
    \bottomrule
    \end{tabular}
    \label{tab:methods}
\end{minipage}
\label{fig:crosspage}
\end{figure*}
\vspace{-1ex}
\section{CBF-RL for Safe Humanoid Locomotion}
\vspace{-1pt}
In the following section, we present two different use cases of CBF-RL in humanoid locomotion to show the generality and performance of our method. Each uses a different set of user-specified reduced-order coordinates and a valid CBF for the resulting reduced-order dynamics.
The nominal rewards and observations (history length 5) follow \cite{zakka2025mujoco}.
Note that for blind stair climbing, we employ the asymmetric actor-critic method \cite{pinto2017asymmetric} and provide a height scan of size 1m $\times$ 1.5m with a resolution of 0.1m to the critic observation with a history length of 1.
We use an MLP policy with 3 hidden layers of $[512, 256, 128]$ neurons running at $50$Hz, outputting 
the joint position setpoints of the 12-DoF lower body for the Unitree G1. 
We train in IsaacLab with 4096 environments with $\Delta t = 0.005$s. Each episode lasts for a maximum of 20,000 steps on the NVIDIA RTX 4090 
GPU and perform zero-shot hardware transfer experiments to verify our approach.
\vspace{-1ex}
\subsection{Task Definition}
\vspace{-1ex}
\newsec{Planar Obstacle Avoidance} First, we consider the task where a humanoid robot needs to avoid obstacles during locomotion without  intervention, even when the velocity command is to collide with the obstacle. 
Thus, we can simplify the safety problem as a single integrator problem where the policy modulates the robot's planar velocities $\mb{v}^{\mathrm{base}}_{\mathrm{planar}}=[v_x,v_y]$ to maintain safe distances from the closest obstacle, a cylinder centered at $\mb{p}^r_o$ in the robot frame.
We define the safety function 
$
    h(\mb{p}) = ||\mb{p}^r_o|| - R_r - R_o
$\
where $R_r$ and $R_o$ are the radii of the robot and the obstacle respectively.
We then train our robot with the CBF reward:
\vspace{-1ex}
\begin{align} 
    r_\mathrm{obstacle\, cbf} = &\min(\tfrac{\mb{p}^r_o}{||\mb{p}^r_o||}^\top\, \mb{v}_{\mathrm{planar}}
+ \alpha h(\mb{p}),0) \nonumber\\&+ \exp\!\left(-\tfrac{\|\mb{v}^{\mathrm{base}}_{\mathrm{planar}} - \mb{v}^{\mathrm{safe}}_{\mathrm{planar}}\|^2}{\sigma^2}\right) - 1.
\end{align} 

\begin{figure*}[htbp]
  \centering
\includegraphics[width=0.88\linewidth]{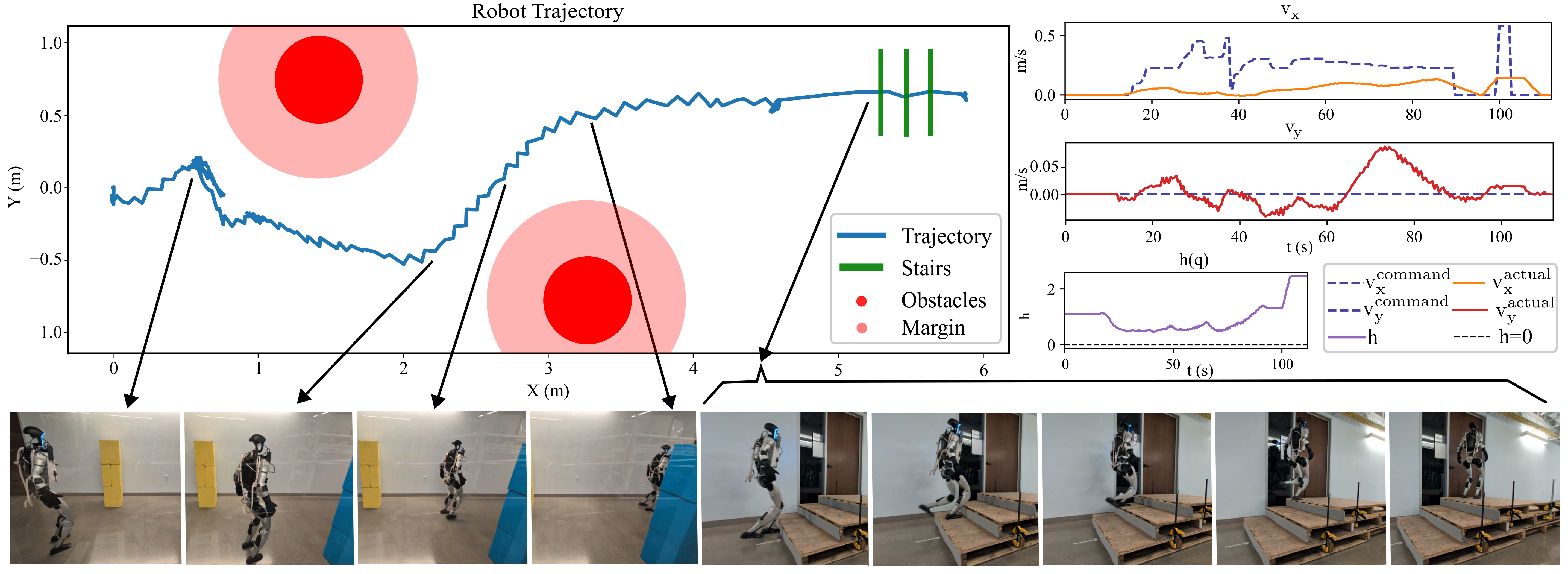}
 \vspace{-1ex}
  \caption{Robot trajectory, $h$, and command vs. actual velocity visualization. The robot avoids obstacles approximated as cylinders without a runtime safety filter and climb up stairs. The velocity plots show the robot modulating its own velocities despite the command and the $h$ plot quantifies safety.}
  \label{fig:dummy}
  \vspace{-1em}
\end{figure*}
\begin{figure*}
    \centering
   \includegraphics[width=0.88\linewidth]{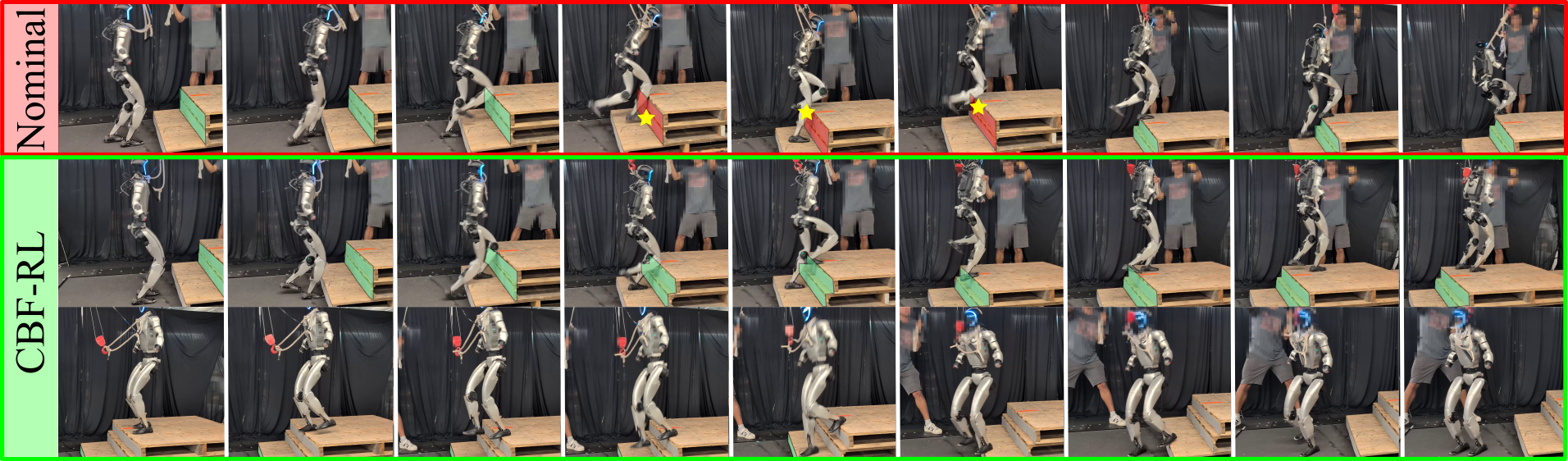}
    \caption{Snapshots of high stairs of riser height 0.3m. The nominal policy clips its feet against the riser and stumbles, as shown with the red CBF violations while the CBF-RL dual trained policy successfully climbs up and down. The yellow star marks the point the robot's feet collides with the stair riser.}
    \vspace{-1em}
    \label{fig:inside-stairs}
\end{figure*}

\begin{figure*}
    \centering
   \includegraphics[width=0.88\linewidth]{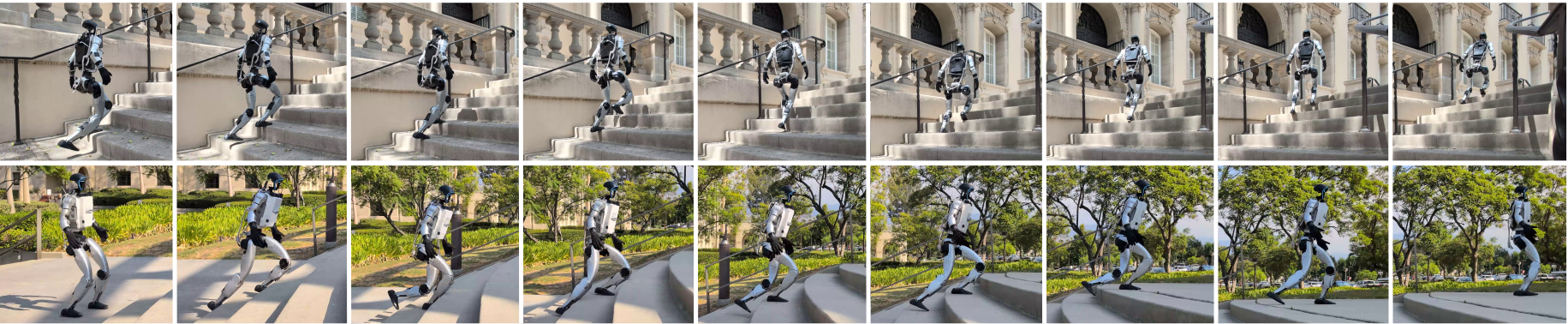}
   
    \caption{Snapshots of outdoor experiments. The robot is able to climb up stairs of varying roughness, tread depths and riser heights.}
    \label{fig:outside-stairs}
    \vspace{-3ex}
\end{figure*}
\vspace{-1ex}
\newsec{Stair Climbing}
Second, we consider the task of humanoid locomotion on stairs. For this task, we consider the kinematic model of the foot as our reduced-order model
$
    \mb{q}^{\mathrm{sw}}_{k+1} = \mb{q}^{\mathrm{sw}}_{k} + \Delta t \, \mathbf{J}^{\mathrm{sw}}(\mb{q}^{\mathrm{sw}}_{k})\mb{v}^{\mathrm{sw}}_k 
$, 
where $\mb{q}^{\mathrm{sw}} = [p_x, p_y]^T$ is the swing foot's position in the body frame, $\mathbf{J}^{\mathrm{sw}}$ comprises the rows of the robot's body Jacobian associated with the foot's position, and $\mb{v}^{\mathrm{sw}}$ is robot's joint velocities. 
In climbing stairs, one problem is the robot hitting its toe against the next stair riser.
We design the barrier as the distance to a hyperplane tangent to the stair after the one it is currently stepping on:
$    h(\mb{q}) = p^{stair}_x - p_x$,
where $p^{stair}_x$ is the $x$ position of the hyperplane in the body reference frame. 
Thus the CBF reward is:
\vspace{-1ex}
\begin{align} 
    r_{\mathrm{next\,stair\,cbf}} =&
\min\!\left(\; \big(- \mathbf{J}^{\mathrm{sw}}_x(\mb{q})\big)\mb{v} + \alpha\,h(\mb{q}),0\right) \nonumber\\ &+
\exp(-\frac{\|\mb{q}-\mb{q}_\mathrm{safe}||^2}{\sigma^2}) - 1.
\end{align} 
added to the nominal rewards including modifications to the feet clearance reward where the reference feet height now is dependent on the stair at the front of the robot and also a penalty on the swing foot force.  

\newsec{Hardware Deployment Considerations}
The stochasticity of real-world state estimation hinders the deployment of runtime explicit filters. CBF-RL is able to internalize safety by learning to map noisy observations to inherently safe actions, thus reducing the reliance of such a filter.

\vspace{-5pt}
\subsection{Hardware Experiments}
\vspace{-5pt}
\newsec{Obstacle Course}
The obstacle course comprises 2 parts: the first part is the obstacle avoidance task, where the robot has to prevent itself from colliding with the obstacles, even if the velocity command intends otherwise;
 the second part comprises stairs constructed out of wooden pallets with a riser height of 0.14m and tread depth of 0.3m. 
During execution, the robot locates the obstacles approximated as cylinders using the ZED 2 RGB-D camera through point cloud clustering.
As seen from the $h$ values in Fig. \ref{fig:dummy}, the robot modulates its own velocity to avoid the obstacle, even when the velocity command prompts it to do so.
However, it has no terrain perception and only uses proprioception for stair climbing. 
Despite this, we observe that the robot uses proprioception to determine when to climb and how high to lift its feet, successfully climbing the wooden stairs.

\vspace{-2pt}
\newsec{Stair Climbing Robustness}
In order to further test the robustness of our method, we experiment both indoors and outdoors.
During indoor experiments, we perform continuous runs up and down stairs and also test the performance of the policy on high stairs, with the dual-trained policy able to climb up stairs 0.3m high and the policy trained without the CBF modifications unable to do so, as shown in Fig. \ref{fig:inside-stairs}.
Here we note that the robot is able to gage the depth and height of the stairs through proprioception and adjust its footsteps accordingly. 
For outdoor experiments, we test the dual-trained policy on concrete-poured stairs of different roughness and sizes, with rougher stairs of riser height 0.14m / tread depth 0.33m and smoother stairs of riser height 0.15m / tread depth 0.4m respectively, as shown in Fig. \ref{fig:outside-stairs}, where the robot could adjust its center of mass by modulating the torso pitch angle to account for deeper and higher stairs.

\section{Conclusion}
This paper introduces CBF-RL, a lightweight dual approach to inject safety into learning by combining training-time CBF safety filtering with reward design, leading to policies that internalize safety and operate without a runtime filter, demonstrating its effectiveness through simulated and real-world experiments.
Looking forward, we plan to incorporate automated barrier discovery, perception-based barriers, and extend the application of CBF-RL beyond locomotion to whole-body loco-manipulation, which will address broader humanoid capabilities.

\newsec{Acknowledgment}
During the preparation of this work, the authors used AI-assisted capabilities (Claude, ChatGPT) for grammar and editing enhancements of the text, as well as for debugging software code related to the experiments. The authors have responsibly reviewed and verified all AI-generated content to ensure correctness.
\balance
\bibliographystyle{IEEEtran}
\bibliography{references}

@inproceedings{ahmadi2019safe,
  title={Safe policy synthesis in multi-agent POMDPs via discrete-time barrier functions},
  author={Ahmadi, Mohamadreza and Singletary, Andrew and Burdick, Joel W and Ames, Aaron D},
  booktitle={2019 IEEE 58th Conference on Decision and Control (CDC)},
  pages={4797--4803},
  year={2019},
  organization={IEEE}
}

@inproceedings{cohen2025safety,
  title={Safety-critical controller synthesis with reduced-order models},
  author={Cohen, Max H and Csomay-Shanklin, Noel and Compton, William D and Molnar, Tamas G and Ames, Aaron D},
  booktitle={2025 American Control Conference (ACC)},
  pages={5216--5221},
  year={2025},
  organization={IEEE}
}

@inproceedings{ames2019control,
  title={Control barrier functions: Theory and applications},
  author={Ames, Aaron D and Coogan, Samuel and Egerstedt, Magnus and Notomista, Gennaro and Sreenath, Koushil and Tabuada, Paulo},
  booktitle={2019 18th European control conference (ECC)},
  pages={3420--3431},
  year={2019},
  organization={Ieee}
}

@inproceedings{li_robust_2020,
  title={Robust model predictive shielding for safe reinforcement learning with stochastic dynamics},
  author={Li, Shuo and Bastani, Osbert},
  booktitle={2020 IEEE International Conference on Robotics and Automation (ICRA)},
  pages={7166--7172},
  year={2020},
  organization={IEEE}
}

@inproceedings{ma_model-based_2021,
  title={Model-based Constrained Reinforcement Learning using Generalized Control Barrier Function},
  author={Ma, Haitong and Chen, Jianyu and Eben, Shengbo and Lin, Ziyu and Guan, Yang and Ren, Yangang and Zheng, Sifa},
  booktitle={2021 IEEE/RSJ International Conference on Intelligent Robots and Systems (IROS)},
  pages={4552--4559},
  year={2021},
  organization={IEEE}
}

@inproceedings{alshiekh_safe_2017,
  title={Safe Reinforcement Learning via Shielding},
  author={Alshiekh, Mohammed and Bloem, Roderick and Ehlers, R{\"u}diger and K{\"o}nighofer, Bettina and Niekum, Scott and Topcu, Ufuk},
  booktitle={32nd AAAI Conference on Artificial Intelligence: AAAI-18},
  pages={2669--2678},
  year={2018}
}

@article{hailemichael_safe_2022,
  title={Safe reinforcement learning for an energy-efficient driver assistance system},
  author={Hailemichael, Habtamu and Ayalew, Beshah and Kerbel, Lindsey and Ivanco, Andrej and Loiselle, Keith},
  journal={IFAC-PapersOnLine},
  volume={55},
  number={37},
  pages={615--620},
  year={2022},
  publisher={Elsevier}
}

@misc{tan_value_2023,
	title = {Value {Functions} are {Control} {Barrier} {Functions}: {Verification} of {Safe} {Policies} using {Control} {Theory}},
	shorttitle = {Value {Functions} are {Control} {Barrier} {Functions}},
	doi = {10.48550/arXiv.2306.04026},
	urldate = {2025-08-29},
	publisher = {arXiv},
	author = {Tan, Daniel C. H. and Acero, Fernando and McCarthy, Robert and Kanoulas, Dimitrios and Li, Zhibin},
	month = dec,
	year = {2023},
	note = {arXiv:2306.04026 [cs]},
}

@article{zhang_control_2025,
	title = {Control {Barrier} {Function}-{Guided} {Deep} {Reinforcement} {Learning} for {Decision}-{Making} of {Autonomous} {Vehicle} at {On}-{Ramp} {Merging}},
	volume = {26},
	copyright = {https://ieeexplore.ieee.org/Xplorehelp/downloads/license-information/IEEE.html},
	issn = {1524-9050, 1558-0016},
	doi = {10.1109/TITS.2025.3540862},
	number = {6},
	urldate = {2025-08-29},
	journal = {IEEE Transactions on Intelligent Transportation Systems},
	author = {Zhang, Changzhu and Dai, Lifei and Zhang, Hao and Wang, Zhuping},
	month = jun,
	year = {2025},
	pages = {8919--8932},
}

@inproceedings{du_reinforcement_2023,
  title={Reinforcement Learning for Safe Robot Control using Control Lyapunov Barrier Functions},
  author={Du, Desong and Han, Shaohang and Qi, Naiming and Ammar, Haitham Bou and Wang, Jun and Pan, Wei},
  booktitle={2023 IEEE International Conference on Robotics and Automation, ICRA 2023},
  pages={9442--9448},
  year={2023},
  organization={IEEE}
}

@inproceedings{cheng_safe_2023,
  title={Safe and efficient reinforcement learning using disturbance-observer-based control barrier functions},
  author={Cheng, Yikun and Zhao, Pan and Hovakimyan, Naira},
  booktitle={Learning for Dynamics and Control Conference},
  pages={104--115},
  year={2023},
  organization={PMLR}
}

@article{cohen_safe_2023,
  title={Safe exploration in model-based reinforcement learning using control barrier functions},
  author={Cohen, Max H and Belta, Calin},
  journal={Automatica},
  volume={147},
  pages={110684},
  year={2023},
  publisher={Elsevier}
}

@inproceedings{zhang_model-based_2021,
	address = {Xi'an, China},
	title = {Model-based {Reinforcement} {Learning} with {Provable} {Safety} {Guarantees} via {Control} {Barrier} {Functions}},
	copyright = {https://ieeexplore.ieee.org/Xplorehelp/downloads/license-information/IEEE.html},
	isbn = {978-1-7281-9077-8},
	doi = {10.1109/ICRA48506.2021.9561253},
	urldate = {2025-08-29},
	booktitle = {2021 {IEEE} {International} {Conference} on {Robotics} and {Automation} ({ICRA})},
	publisher = {IEEE},
	author = {Zhang, Hongchao and Li, Zhouchi and Clark, Andrew},
	month = may,
	year = {2021},
	pages = {792--798},
}

@inproceedings{cheng_end--end_2019,
  title={End-to-end safe reinforcement learning through barrier functions for safety-critical continuous control tasks},
  author={Cheng, Richard and Orosz, G{\'a}bor and Murray, Richard M and Burdick, Joel W},
  booktitle={Proceedings of the AAAI conference on artificial intelligence},
  volume={33},
  number={01},
  pages={3387--3395},
  year={2019}
}

@article{emam_safe_2022,
  title={Safe reinforcement learning using robust control barrier functions},
  author={Emam, Yousef and Notomista, Gennaro and Glotfelter, Paul and Kira, Zsolt and Egerstedt, Magnus},
  journal={IEEE Robotics and Automation Letters},
  year={2022},
  publisher={IEEE}
}

@article{bejarano_safety_2025,
	title = {Safety {Filtering} {While} {Training}: {Improving} the {Performance} and {Sample} {Efficiency} of {Reinforcement} {Learning} {Agents}},
	volume = {10},
	issn = {2377-3766, 2377-3774},
	shorttitle = {Safety {Filtering} {While} {Training}},
	doi = {10.1109/LRA.2024.3512374},
	number = {1},
	urldate = {2025-08-29},
	journal = {IEEE Robotics and Automation Letters},
	author = {Bejarano, Federico Pizarro and Brunke, Lukas and Schoellig, Angela P.},
	month = jan,
	year = {2025},
	note = {arXiv:2410.11671 [cs]},
	pages = {788--795},
}

@inproceedings{nilaksh_barrier_2024,
	title = {Barrier {Functions} {Inspired} {Reward} {Shaping} for {Reinforcement} {Learning}},
	doi = {10.1109/ICRA57147.2024.10610391},
	urldate = {2025-08-29},
	booktitle = {2024 {IEEE} {International} {Conference} on {Robotics} and {Automation} ({ICRA})},
	author = {Nilaksh, Nilaksh and Ranjan, Abhishek and Agrawal, Shreenabh and Jain, Aayush and Jagtap, Pushpak and Kolathaya, Shishir},
	month = may,
	year = {2024},
	note = {arXiv:2403.01410 [cs]},
	pages = {10807--10813},
}

@article{van_wijk_safe_2024,
	title = {Safe {Spacecraft} {Inspection} via {Deep} {Reinforcement} {Learning} and {Discrete} {Control} {Barrier} {Functions}},
	volume = {21},
	issn = {1940-3151, 2327-3097},
	doi = {10.2514/1.I011391},
	language = {en},
	number = {12},
	urldate = {2025-08-30},
	journal = {Journal of Aerospace Information Systems},
	author = {Van Wijk, David and Dunlap, Kyle and Majji, Manoranjan and Hobbs, Kerianne},
	month = dec,
	year = {2024},
	pages = {996--1013},
}

@article{schulman2017proximal,
  title={Proximal policy optimization algorithms},
  author={Schulman, John and Wolski, Filip and Dhariwal, Prafulla and Radford, Alec and Klimov, Oleg},
  journal={arXiv preprint arXiv:1707.06347},
  year={2017}
}

@article{ames2016control,
  title={Control barrier function based quadratic programs for safety critical systems},
  author={Ames, Aaron D and Xu, Xiangru and Grizzle, Jessy W and Tabuada, Paulo},
  journal={IEEE Transactions on Automatic Control},
  volume={62},
  number={8},
  pages={3861--3876},
  year={2016},
  publisher={IEEE}
}

@inproceedings{agrawal2017discrete,
  title={Discrete control barrier functions for safety-critical control of discrete systems with application to bipedal robot navigation.},
  author={Agrawal, Ayush and Sreenath, Koushil},
  booktitle={Robotics: Science and Systems},
  volume={13},
  pages={1--10},
  year={2017},
  organization={Cambridge, MA, USA}
}

@article{breeden2021control,
  title={Control barrier functions in sampled-data systems},
  author={Breeden, Joseph and Garg, Kunal and Panagou, Dimitra},
  journal={IEEE Control Systems Letters},
  volume={6},
  pages={367--372},
  year={2021},
  publisher={IEEE}
}

@article{mittal2023orbit,
   author={Mittal, Mayank and Yu, Calvin and Yu, Qinxi and Liu, Jingzhou and Rudin, Nikita and Hoeller, David and Yuan, Jia Lin and Singh, Ritvik and Guo, Yunrong and Mazhar, Hammad and Mandlekar, Ajay and Babich, Buck and State, Gavriel and Hutter, Marco and Garg, Animesh},
   journal={IEEE Robotics and Automation Letters},
   title={Orbit: A Unified Simulation Framework for Interactive Robot Learning Environments},
   year={2023},
   volume={8},
   number={6},
   pages={3740-3747},
   doi={10.1109/LRA.2023.3270034}
}

@inproceedings{pinto2017asymmetric,
  title={Asymmetric Actor Critic for Image-Based Robot Learning},
  author={Pinto, Lerrel and Andrychowicz, Marcin and Welinder, Peter and Zaremba, Wojciech and Abbeel, Pieter},
  booktitle={14th Robotics: Science and Systems, RSS 2018},
  year={2018},
  organization={MIT Press Journals}
}

@article{krasowski2022provably,
  title={Provably Safe Reinforcement Learning: Conceptual Analysis, Survey, and Benchmarking},
  year={2022},
  author={Krasowski, Hanna and Thumm, Jakob and M{\"u}ller, Marlon and Sch{\"a}fer, Lukas and Wang, Xiao and Althoff, Matthias},
  journal={Transactions on Machine Learning Research}
}

@article{dunlap2023run,
  title={Run time assured reinforcement learning for safe satellite docking},
  author={Dunlap, Kyle and Mote, Mark and Delsing, Kaiden and Hobbs, Kerianne L},
  journal={Journal of Aerospace Information Systems},
  volume={20},
  number={1},
  pages={25--36},
  year={2023},
  publisher={American Institute of Aeronautics and Astronautics}
}

@article{wabersich2021predictive,
  title={A predictive safety filter for learning-based control of constrained nonlinear dynamical systems},
  author={Wabersich, Kim Peter and Zeilinger, Melanie N},
  journal={Automatica},
  volume={129},
  pages={109597},
  year={2021},
  publisher={Elsevier}
}

@article{wang2022ensuring,
  title={Ensuring safety of learning-based motion planners using control barrier functions},
  author={Wang, Xiao},
  journal={IEEE Robotics and Automation Letters},
  volume={7},
  number={2},
  pages={4773--4780},
  year={2022},
  publisher={IEEE}
}

@article{peng2025gait,
  title={Gait-Conditioned Reinforcement Learning with Multi-Phase Curriculum for Humanoid Locomotion},
  author={Peng, Tianhu and Bao, Lingfan and Zhou, CHengxu},
  journal={arXiv preprint arXiv:2505.20619},
  year={2025}
}

@article{truong2025beyondmimic,
  title={BeyondMimic: From Motion Tracking to Versatile Humanoid Control via Guided Diffusion},
  author={Truong, Takara E and Liao, Qiayuan and Huang, Xiaoyu and Tevet, Guy and Liu, C Karen and Sreenath, Koushil},
  journal={arXiv preprint arXiv:2508.08241},
  year={2025}
}

@INPROCEEDINGS{he2025asap, 
    AUTHOR    = {Tairan He AND Jiawei Gao AND Wenli Xiao AND Yuanhang Zhang AND Zi Wang AND Jiashun Wang AND Zhengyi Luo AND Guanqi He AND Nikhil Sobanbabu AND Chaoyi Pan AND Zeji Yi AND Guannan Qu AND Kris Kitani AND Jessica K. Hodgins AND Linxi Fan AND Yuke Zhu AND Changliu Liu AND Guanya Shi}, 
    TITLE     = {{ASAP: Aligning Simulation and Real-World Physics for Learning Agile Humanoid Whole-Body Skills}}, 
    BOOKTITLE = {Proceedings of Robotics: Science and Systems}, 
    YEAR      = {2025}, 
    ADDRESS   = {LosAngeles, CA, USA}, 
    MONTH     = {June}, 
    DOI       = {10.15607/RSS.2025.XXI.066} 
}

@INPROCEEDINGS{zakka2025mujoco, 
    AUTHOR    = {Kevin Zakka AND Baruch Tabanpour AND Qiayuan Liao AND Mustafa Haiderbhai AND Samuel Holt AND Jing Yuan Luo AND Arthur Allshire AND Erik Frey AND Koushil Sreenath AND Lueder Alexander Kahrs AND Carmelo Sferrazza AND Yuval Tassa AND Pieter Abbeel}, 
    TITLE     = {{Demonstrating MuJoCo Playground}}, 
    BOOKTITLE = {Proceedings of Robotics: Science and Systems}, 
    YEAR      = {2025}, 
    ADDRESS   = {LosAngeles, CA, USA}, 
    MONTH     = {June}, 
    DOI       = {10.15607/RSS.2025.XXI.020} 
}

@article{allshire2025visual,
  title={Visual Imitation Enables Contextual Humanoid Control},
  author={Allshire, Arthur and Choi, Hongsuk and Zhang, Junyi and McAllister, David and Zhang, Anthony and Kim, Chung Min and Darrell, Trevor and Abbeel, Pieter and Malik, Jitendra and Kanazawa, Angjoo},
  journal={arXiv preprint arXiv:2505.03729},
  year={2025}
}

@article{su2025hitter,
  title={HITTER: A HumanoId Table TEnnis Robot via Hierarchical Planning and Learning},
  author={Su, Zhi and Zhang, Bike and Rahmanian, Nima and Gao, Yuman and Liao, Qiayuan and Regan, Caitlin and Sreenath, Koushil and Sastry, S Shankar},
  journal={arXiv preprint arXiv:2508.21043},
  year={2025}
}

@article{li2025reinforcement,
  title={Reinforcement learning for versatile, dynamic, and robust bipedal locomotion control},
  author={Li, Zhongyu and Peng, Xue Bin and Abbeel, Pieter and Levine, Sergey and Berseth, Glen and Sreenath, Koushil},
  journal={The International Journal of Robotics Research},
  volume={44},
  number={5},
  pages={840--888},
  year={2025},
  publisher={SAGE Publications Sage UK: London, England}
}

@article{crowley2025optimizing,
  title={Optimizing bipedal locomotion for the 100m dash with comparison to human running},
  author={Crowley, Devin and Dao, Jeremy and Duan, Helei and Green, Kevin and Hurst, Jonathan and Fern, Alan},
  journal={arXiv preprint arXiv:2508.03070},
  year={2025}
}

@article{wang2025omni,
  title={Omni-Perception: Omnidirectional Collision Avoidance for Legged Locomotion in Dynamic Environments},
  author={Wang, Zifan and Ma, Teli and Jia, Yufei and Yang, Xun and Zhou, Jiaming and Ouyang, Wenlong and Zhang, Qiang and Liang, Junwei},
  journal={arXiv preprint arXiv:2505.19214},
  year={2025}
}

@book{pugh2002real,
  title={Real mathematical analysis},
  author={Pugh, Charles Chapman},
  year={2002},
  publisher={Springer}
}

\end{document}